\documentclass[11.5pt]{article}
\usepackage[utf8]{inputenc}
\usepackage{url,ifthen}
\usepackage{srcltx}
\usepackage{multirow}
\usepackage{boxedminipage}
\usepackage[margin=1.1in]{geometry}
\usepackage{nicefrac}
\usepackage{xspace}
\usepackage{graphicx}
\usepackage{srcltx}
\usepackage[usenames]{color}
\usepackage{fullpage}
\usepackage{xspace}
\definecolor{DarkGreen}{rgb}{0.1,0.5,0.1}
\definecolor{DarkRed}{rgb}{0.5,0.1,0.1}
\definecolor{DarkBlue}{rgb}{0.1,0.1,0.5}
\usepackage[small]{caption}
\usepackage{float}
\usepackage{balance}
\usepackage{amsmath}
\usepackage{amsfonts}
\usepackage{amssymb}
\usepackage{amsthm}
\usepackage{mathrsfs}
\usepackage{bbm}

\DeclareMathAlphabet{\mathpzc}{OT1}{pzc}{m}{it}

\newtheorem{theorem}{Theorem}[section]
\newtheorem*{namedtheorem}{\theoremname}
\newcommand{\theoremname}{testing}

\newtheorem{lemma}[theorem]{Lemma}

\newtheorem{claim}[theorem]{Claim}

\newtheorem{corollary}[theorem]{Corollary}

\newtheorem*{question*}{Question}

\theoremstyle{definition}
\newtheorem{definition}[theorem]{Definition}

\theoremstyle{plain}
\newtheorem{Alg}{Algorithm}

\definecolor{DarkGreen}{rgb}{0.1,0.5,0.1}
\definecolor{DarkRed}{rgb}{0.5,0.1,0.1}
\definecolor{DarkBlue}{rgb}{0.1,0.1,0.5}

\usepackage[pdftex]{hyperref}
\hypersetup{
    unicode=false,          
    pdftoolbar=true,        
    pdfmenubar=true,        
    pdffitwindow=false,      
    pdfnewwindow=true,      
    colorlinks=true,       
    linkcolor=DarkGreen,          
    citecolor=DarkGreen,        
    filecolor=DarkGreen,      
    urlcolor=DarkBlue,          
}

\newcommand{\ignore}[1]{}

\renewcommand{\Pr}{\mathop{\bf Pr\/}}                    
\newcommand{\E}{\mathop{\bf E\/}}

\newcommand{\R}{\mathbb R}

\newcommand{\calA}{{\mathcal{A}}}

\newcommand{\calC}{{\mathcal{C}}}

\newcommand{\calK}{\mathcal{K}}

\newcommand{\calR}{\mathcal{R}}

\DeclareMathOperator{\Tr}{Tr}

\makeatletter
\floatstyle{ruled}
\newfloat{fragment}{H}{lop}
\floatname{fragment}{Algorithm}
\renewcommand{\floatc@ruled}[2]{\vspace{2pt}{\@fs@cfont \#1.\:} \#2 \par
 \vspace{1pt}}
\makeatother

\title{{\LARGE Noisy Tensor Completion via the Sum-of-Squares Hierarchy}}
\author{Boaz Barak \thanks{Harvard John A. Paulson School of Engineering and Applied Sciences. Email: {\tt b@boazbarak.org} }
\and
Ankur Moitra \thanks{
Massachusetts Institute of Technology. Department of Mathematics and the Computer Science and Artificial Intelligence Lab. Email: {\tt moitra@mit.edu}.
This work is supported in part by a grant from the MIT NEC Corporation and a Google Research Award. }}

\newcommand{\cA}{\calA}
\newcommand{\cK}{\calK}

\begin{document}
\maketitle

\begin{abstract}
In the noisy tensor completion problem we observe $m$ entries (whose location is chosen uniformly at random) from an unknown $n_1 \times n_2 \times n_3$ tensor $T$. We assume that $T$ is entry-wise close to being rank $r$. Our goal is to fill in its missing entries using as few observations as possible. Let $n = \max(n_1, n_2, n_3)$. We show that if 
$m =  n^{3/2} r$ then there is a polynomial time algorithm based on the sixth level of the sum-of-squares hierarchy for completing it. Our estimate agrees with almost all of $T$'s entries almost exactly and works even when our observations are corrupted by noise. This is also the first algorithm for tensor completion that works in the overcomplete case when $r > n$, and in fact it works all the way up to $r = n^{3/2-\epsilon}$.

Our proofs are short and simple and are based on establishing a new connection between noisy tensor completion (through the language of Rademacher complexity) and the task of refuting random constant satisfaction problems. This connection seems to have gone unnoticed even in the context of matrix completion. Furthermore, we use this connection to show matching lower bounds. Our main technical result is in characterizing the Rademacher complexity of the sequence of norms that arise in the sum-of-squares relaxations to the tensor nuclear norm.  These results point to an interesting new direction: Can we explore computational vs. sample complexity tradeoffs through the sum-of-squares hierarchy?

\end{abstract}

\thispagestyle{empty}

\newpage

\setcounter{page}{1}

\section{Introduction}

Matrix completion is one of the cornerstone problems in machine learning and has a diverse range of applications. One of the original motivations for it comes from the {\em Netflix Problem} where the goal is to predict user-movie ratings based on all the ratings we have observed so far, from across many different users. We can organize this data into a large, partially observed matrix where each row represents a user and each column represents a movie. The goal is to fill in the missing entries. The usual assumptions are that the ratings depend on only a few hidden characteristics of each user and movie and that the underlying matrix is approximately \emph{low rank}. Another standard assumption is that it is incoherent, which we elaborate on later. How many entries of $M$ do we need to observe in order to fill in its missing entries? And are there efficient algorithms for this task?

There have been thousands of papers on this topic and by now we have a relatively complete set of answers. A representative result (building on earlier works by Fazel \cite{Fa}, Recht, Fazel and Parrilo \cite{RFP}, Srebro and Shraibman \cite{SS}, Candes and Recht \cite{CR}, Candes and Tao \cite{CT}) due to Keshavan, Montanari and Oh~\cite{KMO2} can be phrased as follows: Suppose $M$ is an unknown $n_1 \times n_2$ matrix that has rank $r$ but each of its entries has been corrupted by independent Gaussian noise with standard deviation $\delta$. Then if we observe roughly
$$m = (n_1 + n_2) r \log (n_1 + n_2)$$
of its entries, the locations of which are chosen uniformly at random, there is an algorithm that outputs a matrix $X$ that with high probability satisfies
$$\mbox{err}(X) = \frac{1}{n_1 n_2} \sum_{i,j} \Big | X_{i,j}  - M_{i,j} \Big | \leq O(\delta) \;.$$
There are extensions to non-uniform sampling models \cite{LS, CBSW}, as well as various efficiency improvements \cite{JNS, Ha}. What is particularly remarkable about these guarantees is that the number of observations needed is within a logarithmic factor of the number of parameters \---- $(n_1 +n_2)r$ \---- that define the model.

In fact, there are benefits to working with even higher-order structure but so far there has been little progress on natural extensions to the tensor setting. To motivate this problem, consider the {\em Groupon Problem} (which we introduce here to illustrate this point) where the goal is to predict user-activity ratings. The challenge is that which activities we should recommend (and how much a user liked a given activity)  depends on {\em time} as well \---- weekday/weekend, day/night, summer/fall/winter/spring, etc. or even some combination of these. As above, we can cast this problem as a large, partially observed tensor where the first index represents a user, the second index represents an activity and the third index represents the time period. It is again natural to model it as being close to low rank, under the assumption that a much smaller number of (latent) factors about the interests of the user, the type of activity and the time period should contribute to the rating. How many entries of the tensor do we need to observe in order to fill in its missing entries? This problem is emblematic of a larger issue: Can we always solve linear inverse problems when the number of observations is comparable to the number of parameters in the mode, or is computational intractability an obstacle?

In fact, one of the advantages of working with tensors is that their decompositions are unique in important ways that matrix decompositions are not. There has been a groundswell of recent work that uses tensor decompositions for exactly this reason for parameter learning in phylogenetic trees \cite{MR}, HMMs \cite{MR}, mixture models \cite{HK}, topic models \cite{AFHKL} and to solve community detection \cite{AGHK}. In these applications, one assumes access to the entire tensor (up to some sampling noise). But given that the underlying tensors are low-rank, can we observe fewer of their entries and still utilize tensor methods?

A wide range of approaches to solving tensor completion have been proposed \cite{LMWY, GRY, SDS, THK, MHWG, KSV, JO, BSa, YZ}. However, in terms of provable guarantees none\footnote{Most of the existing approaches rely on computing the tensor nuclear norm, which is hard to compute \cite{Gu, HM}. The only other algorithms we are aware of \cite{JO, BSa} require that the factors be orthogonal. This is a rather strong assumption. First, orthogonality requires the rank to be at most $n$. Second, even when $r\leq n$, most tensors need to be ``whitened'' to be put in this form and then a random sample from the ``whitened" tensor would correspond to a (dense) linear combination of the entries of the original tensor, which would be quite a different sampling model. } of them improve upon the following n\"aive algorithm. If the unknown tensor $T$ is $n_1 \times n_2 \times n_3$ we can treat it as a collection of $n_1$ matrices each of size $n_2 \times n_3$. It is easy to see that if $T$ has rank at most $r$ then each of these slices also has rank at most $r$ (and they inherit incoherence properties as well). By treating a third-order tensor as nothing more than an {\em unrelated} collection of $n_1$ low-rank matrices, we can complete each slice separately using roughly $m =  n_1 (n_2 + n_3) r \log (n_2 + n_3)$
observations in total. When the rank is constant, this is a {\em quadratic} number of observations even though the number of parameters in the model is {\em linear}. 

Here we show how to solve the (noisy) tensor completion problem with many fewer observations. Let $n_1 \leq n_2 \leq n_3$. We give an algorithm based on the sixth level of the sum-of-squares hierarchy that can accurately fill in the missing entries of an unknown, incoherent $n_1 \times n_2 \times n_3$ tensor $T$ that is entry-wise close to being rank $r$ with roughly
$$m = (n_1)^{1/2} (n_2 + n_3) r \log^4 (n_1 + n_2 + n_3)$$
observations. Moreover, our algorithm works even when the observations are corrupted by noise. When $n = n_1 = n_2 = n_3$, this amounts to about $n^{1/2} r $ observations per slice which is much smaller than what we would need to apply matrix completion on each slice separately. Our algorithm needs to leverage the structure between the various slices.

\subsection{Our Results}

We give an algorithm for noisy tensor completion that works for third-order tensors. Let $T$ be a third-order $n_1 \times n_2 \times n_3$ tensor that is entry-wise close to being low rank. In particular let
\begin{equation}\label{eq:tensor}
T = \sum_{\ell = 1}^r \sigma_\ell \mbox{ } a_\ell \otimes b_\ell \otimes c_\ell + \Delta
\end{equation}
where $\sigma_\ell$ is a scalar and $a_\ell, b_\ell$ and $c_\ell$ are vectors of length $n_1$, $n_2$ and $n_3$ respectively. Here $\Delta$ is a tensor that represents noise. Its entries can be thought of as representing model misspecification because $T$ is not exactly low rank or noise in our observations or both. We will only make assumptions about the average and maximum absolute value of entries in $\Delta$. The vectors $a_\ell, b_\ell$ and $c_\ell$ are called {\em factors}, and we will assume that their norms are roughly $\sqrt{n_i}$ for reasons that will become clear later. Moreover we will assume that the magnitude of each of their entries is bounded by $C$ in which case we call the vectors $C$-incoherent\footnote{Incoherence is often defined based on the span of the factors, but we will allow the number of factors to be larger than any of the dimensions of the tensor so we will need an alternative way to ensure that the non-zero entries of the factors are spread out}. (Note that a random vector of dimension $n$ and norm $\sqrt{n}$ will be $O(\sqrt{\log n_i})$-incoherent with high probability.)
The advantage of these conventions are that a typical entry in $T$ does not become vanishingly small as we increase the dimensions of the tensor. This will make it easier to state and interpret the error bounds of our algorithm.

 Let $\Omega$ represent the locations of the entries that we observe, which (as is standard) are chosen uniformly at random and without replacement. Set $|\Omega| = m$. Our goal is to output a hypothesis $X$ that has small entry-wise error, defined as:
$$\mbox{err}(X) = \frac{1}{n_1 n_2 n_3} \sum_{i,j,k} \Big | X_{i,j,k}  - T_{i,j,k} \Big |$$
This measures the error on both the observed and unobserved entries of $T$. Our goal is to give algorithms that achieve {\em vanishing} error, as the size of the problem increases. Moreover we will want algorithms that need as few observations as possible. 
 Here and throughout let $n_1 \leq n_2 \leq n_3$ and $n = \max\{ n_1,n_2,n_3\}$. Our main result is:

\begin{theorem}[Main theorem] \label{thm:predict}
Suppose we are given $m$ observations whose locations are chosen uniformly at random (and without replacement) from a tensor $T$ of the form (\ref{eq:tensor}) where each of the factors $a_\ell, b_\ell$ and $c_\ell$ are $C$-incoherent. Let $\delta = \frac{1}{n_1 n_2 n_3} \sum_{i,j,k} | \Delta_{i,j,k}|$. And let $r^* = \sum_{\ell = 1}^r |\sigma_\ell|$. Then there is a polynomial time algorithm that outputs a hypothesis $X$ that with probability $1- \epsilon$ satisfies
$$\mbox{err}(X) \leq 4 C^3 r^* \sqrt{\frac{ (n_1)^{1/2} (n_2 + n_3) \log^4 n + \log 2/\epsilon}{ m } }\ + 2 \delta$$
provided that $\max_{i,j,k} |\Delta_{i,j,k}| \leq \sqrt{\frac{m}{\log 2/\epsilon}} \delta$.
\end{theorem}

Since the error bound above is quite involved, let us dissect the terms in it. In fact, having an additive $\delta$ in the error bound is unavoidable. We have not assumed anything about $\Delta$ in (\ref{eq:tensor}) except a bound on the average and maximum magnitude of its entries. If $\Delta$ were a random tensor whose entries are $+\delta$ and $-\delta$ then no matter how many entries of $T$ we observe, we cannot hope to obtain error less than $\delta$ on the unobserved entries\footnote{The factor of $2$ is not important, and comes from needing a bound on the empirical error of how well the low rank part of $T$ itself agrees with our observations so far. We could replace it with any other constant factor that is larger than $1$.}. The crucial point is that the remaining term in the error bound becomes $o(1)$ when $m = \widetilde{\Omega}((r^*)^2 n^{3/2})$ which for polylogarithmic $r^*$ improves over the n\"aive algorithm for tensor completion by a {\em polynomial} factor in terms of the number of observations. Moreover our algorithm works without any constraints that factors $a_\ell, b_\ell$ and $c_\ell$ be orthogonal or even have low inner-product. 

In non-degenerate cases we can even remove another factor of $r^*$ from the number of observations we need. Suppose that $T$ is a tensor as in (\ref{eq:tensor}), but let $\sigma_\ell$ be Gaussian random variables with mean zero and variance one. The factors $a_\ell, b_\ell$ and $c_\ell$ are still fixed, but because of the randomness in the coefficients $\sigma_\ell$, the entries of $T$ are now random variables. 

\begin{corollary} \label{corr:inf2}
Suppose we are given $m$ observations whose locations are chosen uniformly at random (and without replacement) from a tensor $T$ of the form (\ref{eq:tensor}), where each coefficient $\sigma_\ell$ is a Gaussian random variable with mean zero and variance one, and each of the factors $a_\ell, b_\ell$ and $c_\ell$ are $C$-incoherent. 

Further, suppose that for a $1-o(1)$ fraction of the entries of $T$, we have $\operatorname{var}(T_{i,j,k}) \geq r/\operatorname{polylog}(n) = V$ and that $\Delta$ is a tensor where each entry is a Gaussian with mean zero and variance $o(V)$. Then there is a polynomial time algorithm that outputs a hypothesis $X$ that satisfies
 \[
 X_{i,j,k} = \Big (1\pm o(1) \Big )T_{i,j,k}
 \]
for a $1 - o(1)$ fraction of the entries. The algorithm succeeds with probability at least $1 - o(1)$ over the randomness of the locations of the observations, and the realizations of the random variables $\sigma_\ell$ and the entries of $\Delta$. Moreover the algorithm uses $m = C^6 n^{3/2} r \operatorname{polylog}(n)$ observations. 
\end{corollary}

\noindent In the setting above, it is enough that the coefficients $\sigma_\ell$ are random and that the non-zero entries in the factors are spread out to ensure that the typical entry in $T$ has variance about $r$. Consequently, the typical entry in $T$ is about $\sqrt{r}$. This fact combined with the error bounds in Theorem~\ref{thm:predict} immediately yield the above corollary . Remarkably, the guarantee is interesting even when $r = n^{3/2 - \epsilon}$ (the so-called overcomplete case). In this setting, if we observe a subpolynomial fraction of the entries of $T$ we are able to recover almost all of the remaining entries almost entirely, even though there are no known algorithms for decomposing an overcomplete, third-order tensor even if we are given {\em all} of its entries, at least without imposing much stronger conditions that the factors be nearly orthogonal \cite{GM}. 

We believe that this work is a natural first step in designing practically efficient algorithms for tensor completion. Our algorithms manage to leverage the structure across the slices through the tensor, instead of treating each slice as an independent matrix completion problem. Now that we know this is {\em possible}, a natural follow-up question is to get more efficient algorithms. Our algorithms are based on the sixth level of the sum-of-squares hierarchy and run in polynomial time, but are quite far from being practically efficient as stated. Recent work of Hopkins et al. \cite{HSSS} shows how to speed up sum-of-squares and obtain {\em nearly linear time} algorithms for a number of problems where the only previously known algorithms ran in a prohibitively large degree polynomial running time. Another approach would be to obtain similar guarantees for alternating minimization. Currently, the only known approaches \cite{JO} require that the factors are orthonormal and only work in the undercomplete case. Finally, it would be interesting to get algorithms that recover a low rank tensor exactly when there is no noise.

\subsection{Our approach}

All of our algorithms are based on solving the following optimization problem:
\begin{equation}\label{eq:conv}
 \qquad \min \|X\|_\calK \mbox{ s.t. } \exists X \mbox{ with }  \frac{1}{m} \sum_{(i,j,k) \in \Omega} | X_{i,j,k} - T_{i,j,k}| \leq 2 \delta
 \end{equation}
and outputting the minimizer $X$, where $\|\cdot\|_\calK$ is some norm that can be computed in polynomial time. It will be clear from the way we define the norm that the low rank part of $T$ will itself be a good candidate solution. But this is not necessarily the solution that the convex program finds. How do we know that whatever it finds not only has low entry-wise error on the observed entries of $T$, but also on the unobserved entries too?

 This is a well-studied topic in statistical learning theory, and as is standard we can use the notion of Rademacher complexity as a tool to bound the error. The Rademacher complexity is a property of the norm we choose, and our main innovation is to use the sum-of-squares hierarchy to suggest a suitable norm. Our results are based on establishing a connection between noisy tensor completion and refuting random constraint satisfaction problems. Moreover, our analysis follows by embedding algorithms for refutation within the sum-of-squares hierarchy as a method to bound the Rademacher complexity. 

A natural question to ask is: Are there other norms that have even better Rademacher complexity than the ones we use here, and that are still computable in polynomial time? 
It turns out that {\em any} such norm would immediately lead to much better algorithms for refuting random constraint satisfaction problems than we currently know. We have not yet introduced Rademacher complexity yet, so we state our lower bounds informally:

\begin{theorem} [informal]
For any $\epsilon > 0$, if there is a polynomial time algorithm that achieves error
 $$\mbox{err}(X) \leq r^* \sqrt{\frac{n^{3/2-\epsilon}}{ m } }$$ through the framework of Rademacher complexity then there is an efficient algorithm for refuting a random $3$-SAT formula on $n$ variables with $m = n^{3/2 - \epsilon}$ clauses. Moreover the natural sum-of-squares relaxation requires at least $n^{2\epsilon}$-levels in order to achieve the above error (again through the framework of Rademacher complexity).
\end{theorem}

\noindent These results follow directly from the works of Grigoriev \cite{G}, Schoenebeck \cite{Sch} and Feige \cite{Fe}. There are similar connections between our upper bounds and the work of Coja-Oghlan, Goerdt and Lanka \cite{COGL} who give an algorithm for strongly refuting random $3$-SAT. In Section~\ref{sec:connect} we explain some preliminary connections between these fields, at which point we will be in a better position to explain how we can borrow tools from one area to address open questions in another. We state this theorem more precisely in Corollary~\ref{cor:3sat} and Corollary~\ref{corr:3lb}, which provide both conditional and unconditional lower bounds that match our upper bounds.

\subsection{Computational vs. Sample Complexity Tradeoffs}

It is interesting to compare the story of matrix completion and tensor completion. In matrix completion, we have the best of both worlds: There are efficient algorithms which work when the number of observations is close to the information theoretic minimum. In tensor completion, we gave algorithms that improve upon the number of observations needed by a polynomial factor but still require a polynomial factor more observations than can be achieved if we ignore computational considerations. 
We believe that for many other linear inverse problems (e.g. sparse phase retrieval), there may well be gaps between what can be achieved information theoretically and what can be achieved with computationally efficient estimators. Moreover, proving lower bounds against the sum-of-squares hierarchy offers a new type of evidence that problems are hard, that does not rely on reductions from other average-case hard problems which seem (in general) to be brittle and difficult to execute while preserving the {\em naturalness} of the input distribution. In fact, even when there are such reductions \cite{BR}, the sum-of-squares hierarchy offers a methodology to make sharper predictions for questions like: Is there a quasi-polynomial time algorithm for sparse PCA, or does it require exponential time? 


\subsection*{Organization}

In Section~\ref{sec:connect} we introduce Rademacher complexity, the tensor nuclear norm and strong refutation. We connect these concepts by showing that any norm that can be computed in polynomial time and has good Rademacher complexity yields an algorithm for strongly refuting random $3$-SAT. In Section~\ref{sec:using} we show how a particular algorithm for strong refutation can be embedded into the sum-of-squares hierarchy and directly leads to a norm that can be computed in polynomial time and has good Rademacher complexity. In Section~\ref{sec:spectral} we establish certain spectral bounds that we need, and prove our main upper bounds. In Section~\ref{sec:lb} we prove lower bounds on the Rademacher complexity of the sequence of norms arising from the sum-of-squares hierarchy by a direct reduction to lower bounds for refuting random $3$-XOR. In Appendix~\ref{app:extensions} we give a reduction from noisy tensor completion on asymmetric tensors to symmetric tensors. This is what allows us to extend our analysis to arbitrary order $d$ tensors, but the proofs are essentially identical to those in the $d = 3$ case but more notationally involved so we omit them. 

\section{Noisy Tensor Completion and Refutation}\label{sec:connect}

Here we make the connection between noisy tensor completion and strong refutation explicit. Our first step is to formulate a problem that is a special case of both, and studying it will help us clarify how notions from one problem translate to the other.

\subsection{The Distinguishing Problem}\label{sec:dist}

Here we introduce a problem that we call the {\em distinguishing problem}. We are given random observations from a tensor and promised that the underlying tensor fits into one of the two following categories. We want an algorithm that can tell which case the samples came from, and succeeds using as few observations as possible. The two cases are:

\begin{enumerate}

\item Each observation is chosen uniformly at random (and without replacement) from a tensor $T$ where independently for each entry we set
\begin{equation*}
T_{i,j,k}=
\begin{cases}
a_i a_j a_k &\mbox{ with probability } 7/8\\
\mbox{ } \mbox{ } 1 &\mbox { with probability } 1/16\\
-1 &\mbox{ else}
\end{cases}
\end{equation*}
where $a$ is a vector whose entries are $\pm 1$.

\item Alternatively, each observation is chosen uniformly at random (and without replacement) from a tensor $T$ each of whose entries is independently set to either $+1$ or $-1$ and with equal probability.

\end{enumerate}

\noindent In the first case, the entries of the underlying tensor $T$ are {\em predictable}. It is possible to guess a $15/16$ fraction of them correctly, once we have observed enough of its entries to be able to deduce $a$. And in the second case, the entries of $T$ are completely unpredictable because no matter how many entries we have observed, the remaining entries are still random. Thus we cannot predict any of the unobserved entries better than random guessing. 

Now we will explain how the distinguishing problem can be equivalently reformulated in the language of refutation. We give a formal definition for strong refutation later (Definition~\ref{def:strongref}), but for the time being we can think of it as the task of (given an instance of a constraint satisfaction problem) certifying that there is no assignment that satisfies many of the clauses. We will be interested in $3$-XOR formulas, where there are $n$ variables $v_1, v_2, ..., v_n$ that are constrained to take on values $+1$ or $-1$. Each clause takes the form
$$v_i \cdot v_j \cdot v_k = T_{i,j,k}$$
where the right hand side is either $+1$ or $-1$. The clause represents a parity constraint but over the domain $\{+1, -1\}$ instead of over the usual domain $\mathbb{F}_2$. We have chosen the notation suggestively so that it hints at the mapping between the two views of the problem. Each observation $T_{i,j,k}$ maps to a clause $v_i \cdot v_j \cdot v_k = T_{i,j,k}$ and vice-versa. Thus an equivalent way to formulate the distinguishing problem is that we are given a $3$-XOR formula which was generated in one of the following two ways:

\begin{enumerate}

\item Each clause in the formula is generated by choosing an ordered triple of variables $(v_i, v_j, v_k)$ uniformly at random (and without replacement) and we set
\begin{equation*}
v_i \cdot v_j \cdot v_k=
\begin{cases}
a_i a_j a_k &\mbox{ with probability } 7/8\\
\mbox{ } \mbox{ } 1 &\mbox { with probability } 1/16\\
-1 &\mbox{ else}
\end{cases}
\end{equation*}
where $a$ is a vector whose entries are $\pm 1$. Now $a$ represents a planted solution and by design our sampling procedure guarantees that many of the clauses that are generated are consistent with it. 

\item Alternatively, each clause in the formula is generated by choosing an ordered triple of variables $(v_i, v_j, v_k)$ uniformly at random (and without replacement) and we set
$v_i \cdot v_j \cdot v_k = z_{i,j,k}$
where $z_{i,j,k}$ is a random variable that takes on values $+1$ and $-1$.

\end{enumerate}

\noindent In the first case, the $3$-XOR formula has an assignment that satisfies a $15/16$ fraction of the clauses in expectation by setting $v_i = a_i$. In the second case, any fixed assignment satisfies at most half of the clauses in expectation. Moreover if we are given $\Omega(n \log n)$ clauses, it is easy to see by applying the Chernoff bound and taking a union bound over all possible assignments that with high probability there is no assignment that satisfies more than a $1/2 + o(1)$ fraction of the clauses.

This will be the starting point for the connections we establish between noisy tensor completion and refutation. 
Even in the matrix case these connections seem to have gone unnoticed, and the same spectral bounds that are used to analyze the Rademacher complexity of the nuclear norm \cite{SS} are also used to refute random $2$-SAT formulas \cite{GK}, but this is no accident. 

\subsection{Rademacher Complexity}\label{sec:rad}

Ultimately our goal is to show that the hypothesis $X$ that our convex program finds is entry-wise close to the unknown tensor $T$. By virtue of the fact that $X$ is a feasible solution to (\ref{eq:conv}) we know that it is entry-wise close to $T$ on the observed entries. This is often called the empirical error:

\begin{definition}
For a hypothesis $X$, the empirical error is
$$\mbox{emp-err}(X) = \frac{1}{m} \sum_{(i,j,k) \in \Omega} | X_{i,j,k} - T_{i,j,k}|$$
\end{definition}

Recall that $\mbox{err}(X)$ is the average entry-wise error between $X$ and $T$, over all (observed and unobserved) entries. Also recall that among the candidate $X$'s that have low empirical error, the convex program finds the one that minimizes $\|X\|_\calK$ for some polynomial time computable norm. The way we will choose the norm $\|\cdot\|_\calK$ and our bound on the maximum magnitude of an entry of $\Delta$ will guarantee that the low rank part of $T$ will with high probability be a feasible solution. This ensures that $\|X\|_\calK$ for the $X$ we find is not too large either. One way to bound $\mbox{err}(X)$ is to show that no hypothesis in the unit norm ball can have too large a gap between its error and its empirical error (and then dilate the unit norm ball so that it contains $X$). With this in mind, we define:

\begin{definition}
For a norm $\|\cdot\|_{\calK}$ and a set $\Omega$ of observations, the generalization error is
$$\sup_{\|X\|_\calK \leq 1} \Big | \mbox{err}(X) - \mbox{emp-err}(X) \Big |  $$
\end{definition}

\noindent It turns out that one can bound the generalization error via the Rademacher complexity. 
\begin{definition}\label{def:rad}
Let $\Omega = \{(i_1, j_1, k_1), (i_2, j_2, k_2), ..., (i_m, j_m, k_m)\}$ be a set of $m$ locations chosen uniformly at random (and without replacement) from $[n_1] \times [n_2] \times [n_3]$. And let $\sigma_1, \sigma_2, ..., \sigma_\ell$ be random $\pm 1$ variables.
The Rademacher complexity of (the unit ball of) the norm $\|\cdot\|_\calK$ is defined as
$$R^m(\|\cdot\|_\calK) = \E_{\Omega, \sigma}  \Big [ \sup_{\|X\|_\calK \leq 1} \Big |  \sum_{\ell =1}^m   \sigma_\ell X_{i_\ell, j_\ell, k_\ell}  \Big | \Big ] $$
\end{definition}

It follows from a standard symmetrization argument from empirical process theory \cite{KP, BM} that the Rademacher complexity does indeed bound the generalization error.

\begin{theorem}\label{thm:generalize}
Let $\epsilon\in (0,1)$ and suppose each $X$ with $\|X\|_{\calK} \leq 1$ has bounded loss  \---- i.e.  $|X_{i,j,k} - T_{i,j,k}| \leq a$ and that locations $(i, j, k)$ are chosen uniformly at random and without replacement. Then with probability at least $1- \epsilon$, for every $X$ with $\|X\|_{\calK} \leq 1$, we have
$$\mbox{err}(X) \leq \mbox{emp-err}(X) + 2 R^m(\|\cdot\|_\calK) + 2 a \sqrt{\frac{\ln(1/\epsilon)}{m}}$$
\end{theorem}

We repeat the proof here following \cite{BM} for the sake of completeness but readers familiar with Rademacher complexity can feel free to skip ahead to Definition~\ref{def:Z}. The main idea is to let $\Omega'$ be an independent set of $m$ samples from the same distribution, again without replacement. The expected generalization error is:

\begin{equation}\label{eqn:generalization}
\E_{\Omega} \Big [ \sup_{\|X\|_\calK \leq 1} \Big | \frac{1}{m} \sum_{\ell =1}^m  | X_{i_\ell, j_\ell, k_\ell} - T_{i_\ell, j_\ell, k_\ell}  | - \E_{i,j,k} [ | X_{i,j,k} - T_{i,j,k} | ] \Big | \Big ]\tag{$\ast$}
\end{equation}

Then we can write
\begin{eqnarray*}
(\mbox{\ref{eqn:generalization}}) &=& \E_{\Omega} \Big [ \sup_{\|X\|_\calK \leq 1} \Big | \frac{1}{m} \sum_{\ell =1}^m  | X_{i_\ell, j_\ell, k_\ell} - T_{i_\ell, j_\ell, k_\ell}  | - \frac{1}{m} \E_{\Omega'} [\sum_{\ell =1}^m | X_{i'_\ell, j'_\ell, k'_\ell} - T_{i'_\ell, j'_\ell, k'_\ell}  | ] \Big | \Big ] \\
&\leq& \E_{\Omega, \Omega'} \Big [ \sup_{\|X\|_\calK \leq 1} \Big | \frac{1}{m} \Big ( \sum_{\ell =1}^m  | X_{i_\ell, j_\ell, k_\ell} - T_{i_\ell, j_\ell, k_\ell}  | - | X_{i'_\ell, j'_\ell, k'_\ell} - T_{i'_\ell, j'_\ell, k'_\ell} |  \Big ) \Big | \Big ]
\end{eqnarray*}
where the last line follows by the concavity of $\sup(\cdot)$. Now we can use the Rademacher (random $\pm 1$) variables $\{\sigma_\ell\}_\ell$ and rewrite the right hand side of the above expression as follows:
\begin{eqnarray*}
(\mbox{\ref{eqn:generalization}}) &\leq& \E_{\Omega, \Omega', \sigma} \Big [ \sup_{\|X\|_\calK \leq 1} \Big | \frac{1}{m} \sum_{\ell =1}^m  \sigma_\ell \Big (  | X_{i_\ell, j_\ell, k_\ell} - T_{i_\ell, j_\ell, k_\ell}  | - | X_{i'_\ell, j'_\ell, k'_\ell} - T_{i'_\ell, j'_\ell, k'_\ell} |  \Big ) \Big | \Big ] \\
&\leq&  \E_{\Omega, \Omega', \sigma} \Big [ \sup_{\|X\|_\calK \leq 1} \Big | \frac{1}{m} \sum_{\ell =1}^m  \sigma_\ell | X_{i_\ell, j_\ell, k_\ell} - T_{i_\ell, j_\ell, k_\ell}  |  \Big | + \Big | \frac{1}{m} \sum_{\ell =1}^m  \sigma_\ell | X_{i'_\ell, j'_\ell, k'_\ell} - T_{i'_\ell, j'_\ell, k'_\ell}  |  \Big | \Big ] \\
&\leq& 2 \E_{\Omega, \sigma} \Big [ \sup_{\|X\|_\calK \leq 1} \Big | \frac{1}{m} \Big ( \sum_{\ell =1}^m  \sigma_\ell | X_{i_\ell, j_\ell, k_\ell} - T_{i_\ell, j_\ell, k_\ell}  |  \Big ) \Big | \Big ] \\
&\leq& 2 \E_{\Omega, \sigma} \Big [ \sup_{\|X\|_\calK \leq 1} \Big | \frac{1}{m} \Big ( \sum_{\ell =1}^m  \sigma_\ell \Big ( | X_{i_\ell, j_\ell, k_\ell}| + |T_{i_\ell, j_\ell, k_\ell}  | \Big )  \Big ) \Big | \Big ] \\
&\leq& 2 \E_{\Omega, \sigma} \Big [  \Big | \frac{1}{m} \sum_{\ell =1}^m \sigma_\ell |T_{i_\ell, j_\ell, k_\ell}|     \Big | \Big ] + 2  \E_{\Omega, \sigma} \Big [ \sup_{\|X\|_\calK \leq 1} \Big | \frac{1}{m}  \sum_{\ell =1}^m   \sigma_\ell |X_{i_\ell, j_\ell, k_\ell}|  \Big | \Big ]\\
&=& 2 \E_{\Omega, \sigma} \Big [  \Big | \frac{1}{m} \sum_{\ell =1}^m \sigma_\ell T_{i_\ell, j_\ell, k_\ell}     \Big | \Big ] + 2  \E_{\Omega, \sigma} \Big [ \sup_{\|X\|_\calK \leq 1} \Big | \frac{1}{m}  \sum_{\ell =1}^m   \sigma_\ell X_{i_\ell, j_\ell, k_\ell}  \Big | \Big ]
\end{eqnarray*}
where the second, fourth and fifth inequalities use the triangle inequality. The equality uses the fact that the $\sigma_\ell$'s are random signs and hence can absorb the absolute value around the terms that they multiply. The second term above in the last expression is exactly the Rademacher complexity that we defined earlier. This argument only shows that the Rademacher complexity bounds the expected generalization error. However it turns out that we can also use the Rademacher complexity to bound the generalization error with high probability by applying McDiarmid's inequality. See for example \cite{Bal}. We also remark that generalization bounds are often stated in the setting where samples are drawn i.i.d., but here the locations of our observations are sampled without replacement. Nevertheless for the settings of $m$ we are interested in, the fraction of our observations that are repeats is $o(1)$ \---- in fact it is subpolynomial \---- and we can move back and forth between both sampling models at negligible loss in our bounds.

In much of what follows it will be convenient to think of $\Omega = \{(i_1, j_1, k_1), (i_2, j_2, k_2), ..., (i_m, j_m, k_m)\}$ and $\{\sigma_\ell\}_\ell$ as being represented by a sparse tensor $Z$, defined below.

\begin{definition}\label{def:Z}
Let $Z$ be an $n_1 \times n_2 \times n_3$ tensor such that
$$Z_{i,j,k} = \begin{cases}
0,  \mbox{ if } (i,j,k) \notin \Omega \\
\sum_{\ell \mbox{ s.t. } (i, j, k) = (i_\ell, j_\ell, k_\ell)} \sigma_\ell 
\end{cases} $$
\end{definition}

This definition greatly simplifies our notation.
In particular we have
$$\sum_{\ell = 1}^m \sigma_\ell X_{i_\ell, j_\ell, k_\ell} =\sum_{i,j,k} Z_{i,j,k} X_{i,j,k} = \langle Z, X \rangle$$
where we have introduced the notation $\langle \mbox{ }\cdot\mbox{ },\mbox{ }\cdot\mbox{ } \rangle$ to denote the natural inner-product between tensors. Our main technical goal in this paper will be to analyze the Rademacher complexity of a sequence of successively tighter norms that we get from the sum-of-squares hierarchy, and to derive implications for noisy tensor completion and for refutation from these bounds.

\subsection{The Tensor Nuclear Norm}

Here we introduce the tensor nuclear norm and analyze its Rademacher complexity. Many works have suggested using it to solve tensor completion problems \cite{LMWY, SDS, YZ}. This suggestion is quite natural given that it is based on a similar guiding principle as that which led to $\ell_1$-minimization in compressed sensing and the nuclear norm in matrix completion \cite{Fa}. More generally, one can define the atomic norm for a wide range of linear inverse problems \cite{CRPW}, and the $\ell_1$-norm, the nuclear norm and the tensor nuclear norm are all special cases of this paradigm. Before we proceed, let us first formally define the notion of incoherence that we gave in the introduction.

\begin{definition}\label{def:balanced}
A length $n_i$ vector $a$ is \emph{$C$-incoherent} if $\|a\| = \sqrt{n_i}$ and $\|a\|_\infty \leq C$. 
\end{definition}

Recall that we chose to work with vectors whose typical entry is a constant so that the entries in $T$ do not become vanishingly small as the dimensions of the tensor increase. 
We can now define the tensor nuclear norm\footnote{The usual definition of the tensor nuclear norm has no constraints that the vectors $a$, $b$ and $c$ be $C$-incoherent. However, adding this additional requirement only serves to further restrict the unit norm ball, while ensuring that the low rank part of $T$ (when scaled down) is still in it, since the factors of $T$ are anyways assumed to be $C$-incoherent. This makes it easier to prove recovery guarantees because we do not need to worry about sparse vectors behaving very differently than incoherent ones, and since we are not going to compute this norm anyways this modification will make our analysis easier.}:

\begin{definition}[tensor nuclear norm]\label{def:anorm}
Let $\calA \subseteq \R^{ n_1 \times n_2 \times n_3 }$ be defined as
$$\calA = \Big \{ X \mbox{ s.t. } \exists \mbox{ distribution } \mu \mbox{ on triples of $C$-incoherent vectors with } X_{i,j,k} = \E_{(a,b,c) \leftarrow \mu}[a_i b_j c_k]\Big \}$$
The {\em tensor nuclear norm} of $X$ which is denoted by $\|X\|_{\calA}$  is the infimum over $\alpha$ such that $X/\alpha \in \cA$.
\end{definition}

\noindent In particular $\|T - \Delta \|_{\calA} \leq r^*$. Finally we give an elementary bound on the Rademacher complexity of the tensor nuclear norm. Recall that $n = \max(n_1, n_2, n_3)$. 

\begin{lemma}\label{radnuc}
$R^m(\| \cdot\|_\calA) = O(C^3\sqrt{\frac{n}{m}})$
\end{lemma}

\begin{proof}
Recall the definition of $Z$ given in Definition~\ref{def:Z}. With this we can write
$$\E_{\Omega, \sigma} \Big [ \sup_{\|X\|_\calA \leq 1} \Big |  \sum_{\ell =1}^m   \sigma_\ell X_{i_\ell, j_\ell, k_\ell}  \Big | \Big ] = \E_{\Omega, \sigma} \Big [\sup_{\mbox{{\small $C$-incoherent} } a, b, c} | \langle Z, a\otimes b \otimes c \rangle| \Big ]$$

 We can now adapt the discretization approach in \cite{FKS}, although our task is considerably simpler because we are constrained to $C$-incoherent $a$'s. In particular, let $$S = \left \{ a \Big | a \mbox{ is $C$-incoherent and } a \in \Big (\epsilon \mathbb{Z} \Big )^n \right \}$$
By standard bounds on the size of an $\epsilon$-net \cite{Mat}, we get that $|S| \leq O(C/\epsilon)^n$. Suppose that $|\langle Z, a\otimes b \otimes c \rangle| \leq M$ for all $a, b, c \in S$. Then for an arbitrary, but $C$-incoherent $a$ we can expand it as $a = \sum_i \epsilon^i a^i$ where each $a^i \in S$ and similarly for $b$ and $c$. And now
$$ |\langle Z, a\otimes b \otimes c \rangle| \leq \sum_i \sum_j \sum_k \epsilon^i \epsilon^j \epsilon^k |\langle Z, a^i \otimes b^i \otimes c^i \rangle| \leq (1-\epsilon)^{-3} M$$
Moreover since each entry in $a \otimes b \otimes c$ has magnitude at most $C^3$ we can apply a Chernoff bound to conclude that for any particular $a, b, c \in S$ we have $$|\langle Z, a\otimes b \otimes c \rangle | \leq O\Big(C^3\sqrt{m \log 1/\gamma}\Big)$$
with probability at least $1-\gamma$. Finally, if we set $\gamma = (\frac{\epsilon}{C})^{-n}$ and we set $\epsilon = 1/2$ we get that
$$R^m(\calA) \leq \frac{(1-\epsilon)^{-3}}{m}  \max_{a, b, c \in S}|\langle Z, a\otimes b \otimes c \rangle |  = O\Big(C^3\sqrt{\frac{n}{m}}\Big)$$
and this completes the proof.
\end{proof}

The important point is that the Rademacher complexity of the tensor nuclear norm is $o(1)$ whenever $m = \omega(n)$. In the next subsection we will connect this to refutation in a way that allows us to strengthen known hardness results for computing the tensor nuclear norm \cite{Gu, HM} and show that it is even hard to compute in an average-case sense based on some standard conjectures about the difficulty of refuting random $3$-SAT.

\subsection{From Rademacher Complexity to Refutation}

Here we show the first implication of the connection we have established. Any norm that can be computed in polynomial time and has good Rademacher complexity immediately yields an algorithm for strongly refuting random $3$-SAT and $3$-XOR formulas. Next let us finally define strong refutation. 

\begin{definition}
For a formula $\phi$, let $\operatorname{opt}(\phi)$ be the largest fraction of clauses that can be satisfied by any assignment.
\end{definition}

In what follows, we will use the term {\em random $3$-XOR formula} to refer to a formula where each clause is generated by choosing an ordered triple of variables $(v_i, v_j, v_k)$ uniformly at random (and without replacement) and setting
$v_i \cdot v_j \cdot v_k = z$
where $z$ is a random variable that takes on values $+1$ and $-1$.

\begin{definition}\label{def:strongref}
An algorithm for strongly refuting random $3$-XOR takes as input a $3$-XOR formula $\phi$ and outputs a quantity $\mbox{alg}(\phi)$ that satisfies
\begin{enumerate}

\item For any $3$-XOR formula $\phi$, $\operatorname{opt}(\phi) \leq \mbox{alg}(\phi)$

\item If $\phi$ is a random $3$-XOR formula with $m$ clauses, then with high probability $\mbox{alg}(\phi) = 1/2 + o(1)$

\end{enumerate}
\end{definition}

\noindent This definition only makes sense when $m$ is large enough so that $\mbox{opt}(\phi) = 1/2 + o(1)$ holds with high probability, which happens when $m = \omega(n)$. The goal is to design algorithms that use as few clauses as possible, and are able to certify that a random formula is indeed far from satisfiable (without underestimating the fraction of clauses that can be satisfied) and to do so as close as possible to the information theoretic threshold.


Now let us use a polynomial time computable norm $\|\cdot\|_\calK$ that has good Rademacher complexity to give an algorithm for strongly refuting random $3$-XOR. As in Section~\ref{sec:dist}, given a formula $\phi$ we map its $m$ clauses to a collection of $m$ observations according to the usual rule: If there are $n$ variables, we construct an $n \times n \times n$ tensor $Z$ where for each clause of the form $v_i \cdot v_j \cdot v_k = z_{i,j,k}$ we put the entry $z_{i,j,k}$ at location $(i, j, k)$. All the rest of the entries in $Z$ are set to zero. We solve the following optimization problem:
\begin{equation} \label{eq:refute}
 \max \eta \mbox{ s.t. } \exists X \mbox{ with } \|X\|_{\calK} \leq 1 \mbox{ and } \frac{1}{m} \langle Z, X \rangle \geq 2 \eta
 \end{equation}
Let $\eta^*$ be the optimum value. We set $\mbox{alg}(\phi) = 1/2 + \eta^*$. What remains is to prove that the output of this algorithm solves the strong refutation problem for $3$-XOR.

\begin{theorem}\label{thm:radtoref}
Suppose that $\|\cdot\|_{\calK}$ is computable in polynomial time and satisfies $\|X\|_\calK \leq 1$ whenever $X = a \otimes a \otimes a$ and $a$ is a vector with $\pm 1$ entries. Further suppose that for any $X$ with $\|X\|_{\calK} \leq 1$ its entries are bounded by $C^3$ in absolute value. Then (\ref{eq:refute}) can be solved in polynomial time and if $R^m(\|\cdot\|_\calK) = o(1)$ then setting $\mbox{alg}(\phi) = 1/2 + \eta^*$ solves strong refutation for $3$-XOR with $O(C^6 m \log n)$ clauses.
\end{theorem}

\begin{proof}
The key observation is the following inequality which relates (\ref{eq:refute}) to $\operatorname{opt}(\phi)$. 
$$ 2\operatorname{opt}(\phi) -1 \leq  \frac{1}{m} \sup_{\|X\|_\calK \leq 1} \langle Z, X \rangle$$
To establish this inequality, let $v_1, v_2, ... , v_n$ be the assignment that maximizes the fraction of clauses satisfied. If we set $a_i = v_i$ and $X = a \otimes a \otimes a$ we have that  $\|X\|_\calK \leq 1$ by assumption. Thus $X$ is a feasible solution. Now with this choice of $X$ for the right hand side, every term in the sum that corresponds to a satisfied clause contributes $+1$ and every term that corresponds to an unsatisfied clause contributes $-1$. We get $ 2\operatorname{opt}(\phi) -1$ for this choice of $X$, and this completes the proof of the inequality above. 

The crucial point is that the expectation of the right hand side over $\Omega$ and $\sigma$ is exactly the Rademacher complexity. However we want a bound that holds with high probability instead of just in expectation. It follows from McDiarmid's inequality and the fact that the entries of $Z$ and of $X$ are bounded by $1$ and by $C^3$ in absolute value respectively that if we take $O(C^6 m \log n)$ observations the right hand side will be $o(1)$ with high probability. In this case, rearranging the inequality we have
$$\operatorname{opt}(\phi) \leq 1/2 + \frac{1}{m} \sup_{\|X\|_\calK \leq 1}  \langle Z, X \rangle $$
The right hand side is exactly $\mbox{alg}(\phi)$ and is $1/2 + o(1)$ with high probability, which implies that both conditions in the definition for strong refutation hold and this completes the proof.
\end{proof}

We can now combine Theorem~\ref{thm:radtoref} with the bound on the Rademacher complexity of the tensor nuclear norm given in Lemma~\ref{radnuc} to conclude that if we could compute the tensor nuclear norm we would also obtain an algorithm for strongly refuting random $3$-XOR with only $m = \Omega(n \log n)$ clauses. It is not obvious but it turns out that any algorithm for strongly refuting random $3$-XOR implies one for $3$-SAT. Let us define strong refutation for $3$-SAT. We will refer to any variable $v_i$ or its negation $\bar{v}_i$ as a literal. We will use the term {\em random $3$-SAT formula} to refer to a formula where each clause is generated by choosing an ordered triple of literals $(y_i, y_j, y_k)$ uniformly at random (and without replacement) and setting
$y_i  \vee y_j \vee y_k = 1$.

\begin{definition}\label{def:strongref2}
An algorithm for strongly refuting random $3$-SAT takes as input a $3$-SAT formula $\phi$ and outputs a quantity $\mbox{alg}(\phi)$ that satisfies
\begin{enumerate}

\item For any $3$-SAT formula $\phi$, $\operatorname{opt}(\phi) \leq \mbox{alg}(\phi)$

\item If $\phi$ is a random $3$-SAT formula with $m$ clauses, then with high probability $\mbox{alg}(\phi) = 7/8 + o(1)$

\end{enumerate}
\end{definition}

The only change from Definition~\ref{def:strongref} comes from the fact that for $3$-SAT a random assignment satisfies a $7/8$ fraction of the clauses in expectation. Our goal here is to certify that the largest fraction of clauses that can be satisfied is $7/8 + o(1)$. The connection between refuting random $3$-XOR and $3$-SAT is often called ``Feige's XOR Trick" \cite{Fe}. The first version of it was used to show that an algorithm for $\epsilon$-refuting $3$-XOR can be turned into an algorithm for $\epsilon$-refuting $3$-SAT. However we will not use this notion of refutation so for further details we refer the reader to \cite{Fe}. The reduction was extended later by Coja-Oghlan, Goerdt and Lanka \cite{COGL} to strong refutation, which for us yields the following corollary:

\begin{corollary}\label{cor:3sat}
Suppose that $\|\cdot\|_{\calK}$ is computable in polynomial time and satisfies $\|X\|_\calK \leq 1$ whenever $X = a \otimes a \otimes a$ and $a$ is a vector with $\pm 1$ entries. Suppose further that for any $X$ with $\|X\|_{\calK} \leq 1$ its entries are bounded by $C^3$ in absolute value and that $R^m(\|\cdot\|_\calK) = o(1)$. Then there is a polynomial time algorithm for strongly refuting a random $3$-SAT formula with $O(C^6 m \log n)$ clauses.
\end{corollary}

Now we can get a better understanding of the obstacles to noisy tensor completion by connecting it to the literature on refuting random $3$-SAT. Despite a long line of work on refuting random $3$-SAT \cite{GK, FGK, FO, FKO, COGL}, there is no known polynomial time algorithm that works with $m = n^{3/2 - \epsilon}$ clauses for any $\epsilon > 0$. Feige \cite{Fe} conjectured that for any constant $C$, there is no polynomial time algorithm for refuting random $3$-SAT with $m = Cn$ clauses\footnote{In Feige's paper \cite{Fe} there was no need to make the conjecture any stronger because it was already strong enough for all of the applications in inapproximability.}. Daniely et al. \cite{DLSS1} conjectured that there is no polynomial time algorithm for $m = n^{3/2 - \epsilon}$ for any $\epsilon > 0$. What we have shown above is that any norm that is a relaxation to the tensor nuclear norm and can be computed in polynomial time but has Rademacher complexity is $R^m(\|\cdot\|_\calK) = o(1)$ for $m = n^{3/2 - \epsilon}$ would disprove the conjecture of Daniely et al. \cite{DLSS1} and would yield much better algorithms for refuting random $3$-SAT than we currently know, despite fifteen years of work on the subject.  

This leaves open an important question. While there are no known algorithms for strongly refuting random $3$-SAT with $m = n^{3/2 - \epsilon}$ clauses, there are algorithms that work with roughly $m = n^{3/2}$ clauses \cite{COGL}. Do these algorithms have any implications for noisy tensor completion? We will adapt the algorithm of Coja-Oghlan, Goerdt and Lanka \cite{COGL} and embed it within the sum-of-squares hierarchy. In turn, this will give us a norm that we can use to solve noisy tensor completion which uses a polynomial factor fewer observations than known algorithms.

\section{Using Resolution to Bound the Rademacher Complexity}\label{sec:using}

\subsection{Pseudo-expectation}

Here we introduce the sum-of-squares hierarchy and will use it (at level six) to give a relaxation to the tensor nuclear norm. This will be the norm that we will use in proving our main upper bounds. First we introduce the notion of a pseudo-expectation operator from \cite{BKS1, BKS2, BS}:

\begin{definition}[Pseudo-expectation \cite{BKS1}]\label{def:pseudo-exp}
Let $k$ be even and let $P_k^{n'}$ denote the linear subspace of all  polynomials of degree at most $k$ on $n'$ variables. A linear operator $\widetilde{\E}:P_k^{n'}\rightarrow \R$ is called a \emph{degree $k$ pseudo-expectation operator} if it satisfies the following conditions:
\begin{description}

\item(1)  $\widetilde{\E}[1] = 1$ ({\em normalization})

\item(2)  $\widetilde{\E}[P^2] \geq 0$, for any degree at most $k/2$ polynomial $P$ ({\em nonnegativity})

\end{description}
Moreover suppose that $p \in P_k^{n'}$ with $\mbox{deg}(p) = k'$. We say that $\widetilde{\E}$ satisfies the constraint $\{p = 0\}$ if $\widetilde{\E}[pq] = 0$ for every $q \in P_{k - k'}^{n'}$. And we say that $\widetilde{\E}$ satisfies the constraint $\{ p \geq 0\}$ if $\widetilde{\E}[p q^2] \geq 0$ for every $q \in P_{\lfloor (k-k')/2 \rfloor}^{n'}$.
\end{definition}

The rationale behind this definition is that if $\mu$ is a distribution on vectors in $\R^{n'}$ then the operator $\widetilde{\E}[p] = \E_{Y \leftarrow \mu}[p(Y)]$ is a degree $d$ pseudo-expectation operator for every $d$ \---- i.e. it meets the conditions of Definition~\ref{def:pseudo-exp}. However the converse is in general not true. We are now ready to define the norm that will be used in our upper bounds:

\begin{definition}[$SOS_k$ norm]\label{def:rnorm}
We let $\calK_k$ be the set of all $X \in  \R^{ n_1 \times n_2 \times n_3}$ such that there exists a degree $k$ pseudo-expectation operator on $P_k^{n_1 + n_2 + n_3}$ satisfying the following polynomial
constraints (where the variables are the $Y^{(a)}_i$'s)
\begin{enumerate}
\item[(a)] $\{ \sum_{i=1}^{n_1} (Y^{(1)}_i)^2 = n_1 \}$, $\{ \sum_{i=1}^{n_2} (Y^{(2)}_i)^2 = n_2 \}$ and $\{ \sum_{i=1}^{n_3} (Y^{(3)}_i)^2 = n_3 \}$
\item[(b)] $\{ (Y^{(1)}_i)^2 \leq C^2 \}$, $\{ (Y^{(2)}_i)^2 \leq C^2 \}$ and $\{ (Y^{(3)}_i)^2 \leq C^2 \}$ for all $i$ and
\item[(c)]  $X_{i,j,k} = \widetilde{\E}[ Y^{(1)}_i Y^{(2)}_j Y^{(3)}_k ]$ for all $i,j$ and $k$.
\end{enumerate}
The {\em $SOS_k$ norm} of $X \in \R^{ n_1 \times n_2 \times n_3}$ which is denoted by $\|X\|_{\calK_k}$ is  the infimum over $\alpha$ such that $X/\alpha \in \cK_k$.
\end{definition}

The constraints in Definition~\ref{def:pseudo-exp} can be expressed as an $O(n^k)$-sized semidefinite program. This implies that given any set of polynomial constraints of the form $\{p = 0\}$, $\{p \geq 0\}$, one can efficiently find a degree $k$ pseudo-distribution satisfying those constraints if one exists. This is often called the {\em degree $k$ Sum-of-Squares algorithm} \cite{Sho, N, Las, P}. Hence we can compute the norm $ \|X\|_{\calK_k}$ of any tensor $X$ to within arbitrary accuracy in polynomial time. And because it is a relaxation to the tensor nuclear norm which is defined analogously but over a distribution on $C$-incoherent vectors instead of a pseudo-distribution over them, we have that $\|X\|_{\calK_k} \leq \|X\|_{\calA} $ for every tensor $X$. Throughout most of this paper, we will be interested in the case $k = 6$.

\subsection{Resolution in $\calK_6$}\label{sec:resolution}

Recall that any polynomial time computable norm with good Rademacher complexity with $m$ observations yields an algorithm for strong refutation with roughly $m$ clauses too. Here we will use an algorithm for strongly refuting random $3$-SAT to guide our search for an appropriate norm. We will adapt an algorithm due to Coja-Oghlan, Goerdt and Lanka \cite{COGL} that strongly refutes random $3$-SAT, and will instead give an algorithm that strongly refutes random $3$-XOR. Moreover each of the steps in the algorithm embeds into the sixth level of the sum-of-squares hierarchy by mapping resolution operations to applications of Cauchy-Schwartz, that ultimately show how the inequalities that define the norm (Definition~\ref{def:rnorm}) can be manipulated to give bounds on its own Rademacher complexity.

Let's return to the task of bounding the Rademacher complexity of $\| \cdot \|_{\calK_6}$. Let $X$ be arbitrary but satisfy $\| X\|_{\calK_6} \leq 1$.  Then there is a degree six pseudo-expectation meeting the conditions of Definition~\ref{def:rnorm}. Using Cauchy-Schwartz we have:
\begin{equation}
\Big ( \langle Z, X \rangle \Big )^2 = \Big ( \sum_{i} \sum_{j,k} Z_{i,j,k} \widetilde{\E}[Y^{(1)}_i Y^{(2)}_j Y^{(3)}_k] \Big )^2 \leq n_1 \Big ( \sum_i \Big ( \sum_{j,k} Z_{i,j,k} \widetilde{\E}[Y^{(1)}_i Y^{(2)}_j Y^{(3)}_k] \Big )^2 \Big )\label{eq:sdgfdh}
\end{equation}

To simplify our notation, we will define the following polynomial
$$Q_{i,Z}(Y^{(2)}, Y^{(3)}) =  \sum_{j,k} Z_{i,j,k} Y^{(2)}_j Y^{(3)}_k $$
which we will use repeatedly.
 If $d$ is even then any degree $d$ pseudo-expectation operator satisfies the constraint $(\widetilde{\E}[p])^2 \leq \widetilde{\E}[p^2]$ for every polynomial $p$ of degree at most $d/2$ (e.g., see Lemma $A.4$ in \cite{BBH}). Hence the right hand side of (\ref{eq:sdgfdh}) can be bounded as:

\begin{equation}
n_1 \Big ( \sum_i \Big ( \widetilde{\E}[Y^{(1)}_i Q_{i,Z}(Y^{(2)}, Y^{(3)})]  \Big )^2 \Big ) \leq n_1 \sum_i   \widetilde{\E}\Big [\Big (Y^{(1)}_i Q_{i,Z}(Y^{(2)}, Y^{(3)}) \Big )^2 \Big]
\label{eq:bound-rad}
\end{equation}

It turns out that bounding the right-hand side of (\ref{eq:bound-rad}) boils down to bounding the spectral norm of the following matrix.

\begin{definition}\label{def:randmat}
Let $A$ be the $n_2 n_3 \times n_2 n_3$ matrix whose rows and columns are indexed over ordered pairs $(j,k')$ and $(j',k)$ respectively, defined as
$$A_{j,k',j',k} = \sum_{i} Z_{i,j,k} Z_{i,j',k'}$$
\end{definition}

We can now make the connection to resolution more explicit: We can think of a pair of observations $Z_{i,j,k}, Z_{i,j',k'}$ as a pair of $3$-XOR constraints, as usual. Resolving them (i.e. multiplying them) we obtain a $4$-XOR constraint $$x_j \cdot x_k \cdot x_{j'} \cdot x_{k'} =  Z_{i,j,k}Z_{i,j',k'}$$ $A$ captures the effect of resolving certain pairs of $3$-XOR constraints into $4$-XOR constraints. The challenge is that the entries in $A$ are not independent, so bounding its maximum singular value will require some care. It is important that the rows of $A$ are indexed by $(j, k')$ and the columns are indexed by $(j', k)$, so that $j$ and $j'$ come from different $3$-XOR clauses, as do $k$ and $k'$, and otherwise the spectral bounds that we will want to prove about $A$ would simply not be true! This is perhaps the key insight in \cite{COGL}.


It will be more convenient to decompose $A$ and reason about its two types of contributions separately. To that end, we let $R$ be the $n_2 n_3 \times n_2 n_3$ matrix whose non-zero entries are of the form
$$R_{j,k,j,k} = \sum_{i} Z_{i,j,k} Z_{i,j,k}$$
and all of its other entries are set to zero. Then let $B$ be the $n_2 n_3 \times n_2 n_3$ matrix whose entries are of the form
$$B_{j,k',j',k} = \begin{cases}
0,  \mbox{ if } j = j' \mbox{ and } k = k'  \\
\sum_{i} Z_{i,j,k} Z_{i,j',k'} \mbox{ else }
\end{cases} $$
By construction we have $A = B + R$. Finally:

\begin{lemma}\label{lemma:schur2}
$$ \sum_i   \widetilde{\E} \Big [\Big (Y^{(1)}_i Q_{i,Z}(Y^{(2)}, Y^{(3)}) \Big )^2 \Big] \leq C^2 n_2 n_3 \|B\| +  C^6 m $$
\end{lemma}

\begin{proof}
The pseudo-expectation operator satisfies $\{ (Y^{(1)}_i)^2 \leq C^2 \}$ for all $i$, and hence we have
$$\sum_i   \widetilde{\E}\Big [\Big (Y_i Q_{i,Z}(Y^{(2)}, Y^{(3)}) \Big )^2 \Big] \leq C^2 \sum_i \widetilde{\E}\Big [\Big ( Q_{i,Z}(Y^{(2)}, Y^{(3)}) \Big )^2 \Big] = C^2 \sum_i \sum_{j,k,j',k'} \widetilde{\E}\Big [ Z_{i,j,k} Z_{i,j',k'} Y^{(2)}_j Y^{(3)}_k Y^{(2)}_{j'} Y^{(3)}_{k'}\Big]$$
Now let $Y^{(2)} \in \R^{n_2}$ be a vector of variables where the $i$th entry is $Y^{(2)}_i$ and similarly for $Y^{(3)}$. Then we can re-write the right hand side as a matrix inner-product:
$$C^2\sum_i \sum_{j,k,j',k'} Z_{i,j,k} Z_{i,j',k'} \widetilde{\E}[ Y^{(2)}_j Y^{(3)}_k Y^{(2)}_{j'} Y^{(3)}_{k'}] = C^2  \langle A, \widetilde{\E}[(Y^{(2)} \otimes Y^{(3)}) (Y^{(2)} \otimes Y^{(3)})^T] \rangle$$
We will now bound the contribution of $B$ and $R$ separately. 
\begin{claim}
$\widetilde{\E}[(Y^{(2)} \otimes Y^{(3)}) (Y^{(2)} \otimes Y^{(3)})^T]$ is positive semidefinite and has trace at most $ n_2 n_3$
\end{claim}
\begin{proof}
It is easy to see that a quadratic form on $\widetilde{\E}[(Y^{(2)} \otimes Y^{(3)}) (Y^{(2)} \otimes Y^{(3)})^T]$ corresponds to $\widetilde{\E}[p^2]$ for some $p \in P_2^{n_2 + n_3}$ and this implies the first part of the claim. Finally $$\mbox{Tr}(\widetilde{\E}[(Y^{(2)} \otimes Y^{(3)}) (Y^{(2)} \otimes Y^{(3)})^T]) = \sum_{j,k} \widetilde{\E}[(Y^{(2)}_j)^2 (Y^{(3)}_k)^2] \leq  n_2 n_3$$
where the last equality follows because the pseudo-expectation operator satisfies the constraints $\{ \sum_{i=1}^{n_2} (Y^{(2)}_i)^2 = n_2 \}$ and $\{ \sum_{i=1}^{n_3} (Y^{(3)}_i)^2 = n_3 \}$.
\end{proof}
Hence we can bound the contribution of the first term as
$C^2 \langle B, \widetilde{\E}[(Y^{(2)} \otimes Y^{(3)}) (Y^{(2)} \otimes Y^{(3)})^T]] \rangle \leq C^2 n_2 n_3 \|B\|$.
Now we proceed to bound the contribution of the second term:
\begin{claim}\label{claim:jk}
$\widetilde{\E}[(Y^{(2)}_j)^2 (Y^{(3)}_k)^2] \leq C^4$
\end{claim}
\begin{proof}
It is easy to verify by direct computation that the following equality holds:
$$C^4 - (Y^{(2)}_j)^2 (Y^{(3)}_k)^2 = \Big (C^2 - (Y^{(2)}_j)^2 \Big ) \Big (C^2 - (Y^{(3)}_k)^2 \Big) + \Big (C^2 - (Y^{(3)}_k)^2\Big ) (Y^{(2)}_j)^2 + \Big (C^2  - (Y^{(2)}_j)^2 \Big )(Y^{(3)}_k)^2$$
Moreover the pseudo-expectation of each of the three terms above is nonnegative, by construction. This implies the claim.
\end{proof}
Moreover each entry in $Z$ is in the set $\{-1, 0, +1\}$ and there are precisely $m$ non-zeros. Thus the sum of the absolute values of all entries in $R$ is at most $m$. Now we have:
$$C^2  \langle R, \widetilde{\E}[(Y^{(2)} \otimes Y^{(3)}) (Y^{(2)} \otimes Y^{(3)})^T] \rangle \leq C^2 \sum_{j,k} R_{j, k, j, k} \widetilde{\E}[(Y^{(2)}_j)^2 (Y^{(3)}_k)^2] \leq C^6 m $$
And this completes the proof of the lemma.
\end{proof}

\section{Spectral Bounds}\label{sec:spectral}

Recall the definition of $B$ given in the previous section. In fact, for our spectral bounds it will be more convenient to relabel the variables (but keeping the definition intact):
$$B_{j,k,j',k'} = \begin{cases}
0,  \mbox{ if } j = j' \mbox{ and } k = k'  \\
\sum_{i} Z_{i,j,k'} Z_{i,j',k} \mbox{ else }
\end{cases} $$
\noindent Let us consider the following random process: For $r = 1, 2, ..., O(\log n)$ partition the set of all ordered triples $(i,j,k)$ into two sets $S_r$ and $T_r$. We will use this ensemble of partitions to define an ensemble of matrices $\{\mathsf{B}^r\}_{r =1}^{O(\log n)}$: Set $U_{i,j,k'}^r$ as equal to $Z_{i,j,k'}$ if $(i,j,k') \in S_r$ and zero otherwise. Similarly set $V_{i,j',k}^r$ equal to $Z_{i,j',k}$ if $(i,j',k) \in T_r$ and zero otherwise. Also let $E_{i,j,j',k,k',r}$ be the event that there is no $r' < r$ where $(i,j,k') \in S_{r'}$ and $(i,j',k) \in T_{r'}$ or vice-versa. Now let
$$\mathsf{B}^r_{j,k,j',k'} = \sum_{i} U_{i,j,k'}^r V_{i,j',k}^r \mathbbm{1}_{E}$$
where $\mathbbm{1}_E$ is short-hand for the indicator function of the event $E_{i,j,j',k,k',r}$. The idea behind this construction is that each pair of triples $(i,j,k')$ and $(i,j',k)$ that contributes to $B$ will be contribute to some $\mathsf{B}^r$ with high probability. Moreover it will not contribute to any later matrix in the ensemble. Hence with high probability $$B = \sum_{r = 1}^{O(\log n)} \mathsf{B}^r$$

Throughout the rest of this section, we will suppress the superscript $r$ and work with a particular matrix in the ensemble, $\mathsf{B}$. Now let $\ell$ be even and consider
$$\Tr(\underbrace{\mathsf{B} \mathsf{B}^T \mathsf{B} \mathsf{B}^T ... \mathsf{B} \mathsf{B}^T}_{\ell \mbox{ times }})$$
As is standard, we are interested in bounding $\E[\Tr(\mathsf{B} \mathsf{B}^T \mathsf{B}\mathsf{B}^T ... \mathsf{B} \mathsf{B}^T)]$ in order to bound $\|\mathsf{B}\|$. But note that $\mathsf{B}$ is {\em not} symmetric. Also note that the random variables $U$ and $V$ are not independent, however whether or not they are non-zero is non-positively correlated and their signs are mutually independent. Expanding the trace above we have
\begin{eqnarray*}
\Tr(\mathsf{B} \mathsf{B}^T \mathsf{B}\mathsf{B}^T ... \mathsf{B} \mathsf{B}^T) &=& \sum_{j_1, k_1} \sum_{j_2, k_2} ... \sum_{j_{\ell-1}, k_{\ell -1}} \mathsf{B}_{j_1,k_1,j_2, k_2} \mathsf{B}_{j_3, k_3, j_2, k_2} ... \mathsf{B}_{j_1,k_1, j_{\ell }, k_{\ell } }\\
&=&  \sum_{j_1, k_1} \sum_{i_1}  \sum_{j_2, k_2} \sum_{i_2} ... \sum_{j_{\ell  }, k_{\ell  }} \sum_{i_\ell} U_{i_1,j_1,k_2} V_{i_1,j_2,k_1}  \mathbbm{1}_{E_1} U_{i_2,j_3,k_2} V_{i_2,j_2,k_3} \mathbbm{1}_{E_2}...  U_{i_\ell,j_{1},k_\ell} V_{i_\ell,j_\ell,k_1}\mathbbm{1}_{E_\ell}
\end{eqnarray*}
where $\mathbbm{1}_{E_1}$ is the indicator for the event that the entry $\mathsf{B}_{j_1,k_1,j_2,k_2}$ is not covered by an earlier matrix in the ensemble, and similarly for $\mathbbm{1}_{E_2}, ..., \mathbbm{1}_{E_\ell}$.

Notice that there are $2 \ell$ random variables in the above sum (ignoring the indicator variables). Moreover if any $U$ or $V$ random variable appears an odd number of times, then the contribution of the term to $\E[\Tr(\mathsf{B} \mathsf{B}^T \mathsf{B}\mathsf{B}^T ... \mathsf{B} \mathsf{B}^T)]$ is zero. We will give an encoding for each term that has a non-zero contribution, and we will prove that it is injective. 

Fix a particular term in the above sum where each random variable appears an even number of times. Let $s$ be the number of distinct values for $i$. Moreover let $i_1, i_2, ..., i_s$ be the order that these indices first appear. Now let $r^j_1$  denote the number of distinct values for $j$  that appear with $i_1$ in $U$ terms \---- i.e. $r^j_1$ is the number of distinct $j$'s that appear as $U_{i_1, j, *}$. Let $r^k_1$ denote the number of distinct values for $k$ that appear with $i_1$ in $U$ terms \---- i.e. $r^k_1$ is the number of distinct $k$'s that appear as or $U_{i_1, *, k}$. Similarly let $q^j_1$  denote the number of distinct values for $j$ that appear with $i_1$ in $V$ terms \---- i.e. $q^j_1$ is the number of distinct $j$'s  that appear as $V_{i_1, j, *}$. And finally let $q^k_1$ denote the number of distinct values for $k$ that appear with $i_1$ in $V$ terms \---- i.e. $q^k_1$ is the number of distinct $k$'s  that appear as $V_{i_1, *, k}$.

We give our encoding below. It is more convenient to think of the encoding as any way to answer the following questions about the term.

\begin{itemize}

\item[(a)] What is the order $i_1, i_2, ..., i_s$ of the first appearance of each distinct value of $i$?

\item[(b)] For each $i$ that appears, what is the order of each of the distinct values of $j$'s and $k$'s that appear along with it in $U$? Similarly, what is the order of each of the distinct values of $j$'s and $k$'s that appear along with it in $V$?

\item[(c)] For each step (i.e. a new variable in the term when reading from left to right),
has the value of $i$ been visited already? Also, has the value for $j$ or $k$ that appears along with $U$ been visited? Has the value for $j$ or $k$ that appears along with $V$ been visited? Note that  whether or not $j$ or $k$ has been visited (together in $U$) depends on what the value of $i$ is, and if $i$ is a new value then the $j$ or $k$ value must be new too, by definition. Finally, if any value has already been visited, which earlier value is it?

\end{itemize}

Let $r_j = r^j_1 + r^j_2 + ... + r^j_s$ and $r_k = r^k_1 + r^k_2 + ... + r^k_s$. Similarly let $q_j = q^j_1 + q^j_2 + ... q^j_s$ and $q_k = q^k_1 + q^k_2 + ... q^k_s$. Then the number of possible answers to \textbf{$(a)$} and \textbf{$(b)$} is at most $n_1^s$ and $n_2^{r_j} n_3^{r_k} n_2^{q_j} n_3^{q_k}$ respectively. It is also easy to see that the number of answers to \textbf{$(c)$} that arise over the sequence of $\ell$ steps is at most $8^\ell (s( r_j + r_k)( q_j + q_k))^\ell$. We remark that much of the work on bounding the maximum eigenvalue of a random matrix is in removing any $\ell^\ell$ type terms, and so one needs to encode re-visiting indices more compactly.  However such terms will only cost us polylogarithmic factors in our bound on $\|B\|$.

It is easy to see that this encoding is injective, since given the answers to the above questions one can simulate each step and recover the sequence of random variables. Next we establish some easy facts that allow us to bound $\E[\Tr(\mathsf{B} \mathsf{B}^T \mathsf{B}\mathsf{B}^T ... \mathsf{B} \mathsf{B}^T)]$.

\begin{claim}\label{claim1}
For any term that has a non-zero contribution to $\E[\Tr(\mathsf{B} \mathsf{B}^T \mathsf{B}\mathsf{B}^T ... \mathsf{B} \mathsf{B}^T)]$, we must have $s \leq \ell/2$ and $r_j + q_j + r_k + q_k \leq \ell$
\end{claim}

\begin{proof}
Recall that there are $2 \ell$ random variables in the product and precisely $\ell$ of them correspond to $U$ variables and $\ell$ of them to $V$ variables. Suppose that $s > \ell/2$. Then there must be at least one $U$ variable and at least one $V$ variable that occur exactly once, which implies that its expectation is zero because the signs of the non-zero entries are mutually independent. Similarly suppose $ r_j + q_j + r_k + q_k > \ell$. Then there must be at least one $U$ or $V$ variable that occurs exactly once, which also implies that its expectation is zero.
\end{proof}

\begin{claim}\label{claim2}
For any valid encoding, $s \leq r_j + q_j$ and $s \leq r_k+ q_k$.
\end{claim}

\begin{proof}
This holds because in each step where the $i$ variable is new and has not been visited before, by definition the $j$ variable is new too (for the current $i$) and similarly for the $k$ variable.
\end{proof}

Finally, if $s, r_j, q_j, r_k$ and $q_k$ are defined as above then for any contributing term
$$U_{i_1,j_1,k_2} V_{i_1,j_2,k_1}  U_{i_2,j_3,k_2} V_{i_2,j_2,k_3} ...  U_{i_\ell,j_{1},l_\ell} V_{i_\ell,j_\ell,k_1}$$
its expectation is at most $p^{r_j + r_k} p^{q_j + q_k}$ where $p = m/n_1 n_2 n_3$ because there are exactly $r_j + r_k$ distinct $U$ variables and $q_j + q_k$ distinct $V$ variables whose values are in the set $\{-1, 0, +1\}$ and whether or not a variable is non-zero is non-positively correlated and the signs are mutually independent.

This now implies the main lemma:

\begin{lemma}\label{lem:walks}
$\E[\Tr(\mathsf{B} \mathsf{B}^T \mathsf{B}\mathsf{B}^T ... \mathsf{B} \mathsf{B}^T)] \leq n_1^{\ell/2} (\max(n_2, n_3))^{\ell} p^\ell (\ell)^{3 \ell + 3}$
\end{lemma}

\begin{proof}
Note that the indicator variables only have the effect of zeroing out some terms that could otherwise contribute to $\E[\Tr(\mathsf{B} \mathsf{B}^T \mathsf{B}\mathsf{B}^T ... \mathsf{B} \mathsf{B}^T)]$.  Returning to the task at hand, we have
$$\E[\Tr(\mathsf{B} \mathsf{B}^T \mathsf{B}\mathsf{B}^T ... \mathsf{B} \mathsf{B}^T)] \leq \sum_{s, r_j, r_k, q_j, q_k} n_1^s n_2^{r_j} n_3^{r_k} n_2^{q_j} n_3^{q_k} p^{r_j + r_k} p^{q_j + q_k} 8^\ell (s( r_j + r_k)( q_j + q_k))^\ell
$$
where the sum is over all valid triples $s, r_j, r_k, q_j, q_k$ and hence $s, r, q \leq \ell/2$ and $s \leq r_j + r_k$ and $s \leq q_j + q_k$ using Claim~\ref{claim1} and Claim~\ref{claim2}. We can upper bound the above as
\begin{eqnarray*}
\E[\Tr(\mathsf{B} \mathsf{B}^T \mathsf{B}\mathsf{B}^T ... \mathsf{B} \mathsf{B}^T)] &\leq& \sum_{s, r_j, r_k, q_j, q_k} n_1^s (pn_2)^{r_j + q_j} (pn_3)^{r_k + q_k} (\ell)^{3 \ell + 3} \\
 &\leq&  \sum_{s, r_j, r_k, q_j, q_k} n_1^s (p \max(n_2, n_3))^{r_j + q_j + r_k + q_k} (\ell)^{3 \ell + 3}
 \end{eqnarray*}
Now if $p \max(n_2, n_3) \leq 1$ then using Claim~\ref{claim2} followed by the first half of Claim~\ref{claim1} we have:
$$\E[\Tr(\mathsf{B} \mathsf{B}^T \mathsf{B}\mathsf{B}^T ... \mathsf{B} \mathsf{B}^T)] \leq n_1^s (p \max(n_2, n_3))^{2s} (\ell)^{3 \ell + 3} \leq n_1^{\ell/2} (p \max(n_2, n_3))^{\ell} (\ell)^{3 \ell + 3}$$
where the last inequality follows because $p n_1^{1/2} \max(n_2, n_3) > 1$. Alternatively if $p \max(n_2, n_3) > 1$ then we can directly invoke the second half of Claim~\ref{claim1} and get:
$$\E[\Tr(\mathsf{B} \mathsf{B}^T \mathsf{B}\mathsf{B}^T ... \mathsf{B} \mathsf{B}^T)] \leq n_1^s (p \max(n_2, n_3))^\ell (\ell)^{3 \ell + 3}  \leq n_1^{\ell/2} (p \max(n_2, n_3))^{\ell} (\ell)^{3 \ell + 3}$$
Hence $\E[\Tr(\mathsf{B} \mathsf{B}^T \mathsf{B}\mathsf{B}^T ... \mathsf{B} \mathsf{B}^T)] \leq n_1^{\ell/2}\max(n_2, n_3)^\ell p^\ell (\ell)^{3 \ell + 3}$ and this completes the proof.
\end{proof}

As before, let $n = \max(n_1, n_2, n_3)$. Then the last piece we need to bound the Rademacher complexity is the following spectral bound:

\begin{theorem}\label{thm:spectral}
With high probability, $\|B\| \leq O\Big( \frac{m \log^4 n}{n_1^{1/2}\min(n_2, n_3)} \Big)$
\end{theorem}

\begin{proof}
We proceed by using Markov's inequality:
$$
\Pr[\|\mathsf{B}\| \geq n_1^{1/2} \max(n_2, n_3) p (2 \ell)^3 ]  = \Pr \Big [\|\mathsf{B}\|^\ell \geq \Big (n_1^{1/2}\max(n_2, n_3) p (2 \ell)^3 \Big )^\ell \Big ] \\
\leq \frac{\E[\Tr(\mathsf{B} \mathsf{B}^T \mathsf{B}\mathsf{B}^T ... \mathsf{B} \mathsf{B}^T)]}{n_1^{\ell/2}\max(n_2, n_3)^\ell p^{\ell} (2 \ell)^{3\ell}} \leq \frac{\ell^3}{2^{3 \ell}}
$$
and hence setting $\ell = \Theta(\log n)$ we conclude that $\|\mathsf{B}\| \leq 8 n_1^{1/2} \max(n_2, n_3) p \log^3 n$ holds with high probability. Moreover $B = \sum_{r =1}^{O(\log n)} \mathsf{B}^r$ also holds with high probability. If this equality holds and each $\mathsf{B}^r$ satisfies $\|\mathsf{B}^r\| \leq 8 n_1^{1/2}\max(n_2, n_3) p \log^3 n$, we have
$$\|B\| \leq \max_r O(\|\mathsf{B}^r\| \log n) = O\Big( \frac{m \log^4 n}{n_1^{1/2}\min(n_2, n_3)} \Big)$$ where we have used the fact that $p = m/n_1 n_2 n_3$. This completes the proof of the theorem.
\end{proof}

\subsection*{Proofs of Theorem~\ref{thm:predict} and Corollary~\ref{corr:inf2}}

We can now bound the Rademacher complexity of the norm that we get from the six level sum-of-squares relaxation to the tensor nuclear norm:

\begin{theorem}\label{thm:k6rad}
$R^m(\|\cdot\|_{\calK_6}) \leq O\Big(\sqrt{\frac{ (n_1)^{1/2} (n_2 + n_3) \log^4 n}{ m} }\Big)$
\end{theorem}

\begin{proof}  Consider any $X$ with $\|X\|_{\calK_6} \leq 1$. Then using Lemma~\ref{lemma:schur2} and Theorem~\ref{thm:spectral} we have
$$\Big ( \langle Z, X \rangle \Big )^2  \leq n_1 \Big ( \sum_i \Big ( \sum_{j,k} Z_{i,j,k} X_{i,j,k} \Big )^2 \Big ) \leq C^2 n_1 n_2 n_3 \|B\| + C^6 m n_1  = O\Big( m n_1^{1/2} \max(n_2, n_3) \log^4 n + m n_1 \Big) $$
Recall that $Z$ was defined in Definition~\ref{def:Z}. The Rademacher complexity can now be bounded as
$$\frac{1}{m} (\langle Z, X \rangle ) \leq O\Big(\sqrt{\frac{ (n_1)^{1/2} (n_2 + n_3) \log^4 n}{ m } }\Big)$$
which completes the proof of the theorem. 
\end{proof}

Recall that bounds on the Rademacher complexity readily imply bounds on the generalization error (see Theorem~\ref{thm:generalize}). We can now prove Theorem~\ref{thm:predict}:

\begin{proof}
We solve (\ref{eq:conv}) using the norm $\| \cdot \|_{\calK_6}$. Since this norm comes from the sixth level of the sum-of-squares hierarchy, it follows that (\ref{eq:conv}) is an $n^6$-sized semidefinite program and there is an efficient algorithm to solve it to arbitrary accuracy. Moreover we can always plug in $X = T - \Delta$ and the bounds on the maximum magnitude of an entry in $\Delta$ together with the Chernoff bound imply that with high probability $X = T - \Delta$ is a feasible solution. Moreover $\|T - \Delta\|_{\calK_6} \leq r^*$. Hence with high probability, the minimizer $X$ satisfies $\|X\|_{\calK_6} \leq r^*$. Now if we take any such $X$ returned by the convex program, because it is feasible its empirical error is at most $2 \delta$. And since $\|X\|_{\calK_6} \leq r^*$ the bounds on the Rademacher complexity (Theorem~\ref{thm:k6rad}) together with Theorem~\ref{thm:generalize} give the desired bounds on $\mbox{err}(X)$ and complete the proof of our main theorem.
\end{proof}

Finally we prove Corollary~\ref{corr:inf2}:

\begin{proof}
Our goal is to lower bound the absolute value of a typical entry in $T$. To be concrete, suppose that $\operatorname{var}(T_{i,j,k}) \geq f(r, n)$ for a $1-o(1)$ fraction of the entries where $f(r, n) = r^{1/2}/ \log^D n$.
Consider $T_{i,j,k}$, which we will view as a degree three polynomial in Gaussian random variables. Then the anti-concentration bounds of Carbery and Wright \cite{CW} now imply that $|T_{i,j,k}| \geq f(r, n)/\log n $ with probability $1 - o(1)$. With this in mind, we define
$$ \calR = \{ (i, j, k) \mbox{ s.t. } |T_{i,j,k}| \geq f(r, n)/\log n\}$$
and it follows form  Markov's bound that that $|\calR| \geq (1 - o(1)) n_1 n_2 n_3$. Now consider just those entries in $\calR$ which we get substantially wrong:
$$\calR' = \{ (i, j, k) \mbox{ s.t. } (i, j, k) \in \calR \mbox{ and } |X_{i,j,k} - T_{i,j,k}| \geq 1/\log n\}$$
We can now invoke Theorem~\ref{thm:predict} which guarantees that the hypothesis $X$ that results from solving (\ref{eq:conv}) satisfies $\mbox{err}(X) = o(1/\log n)$ with probability $1-o(1)$ provided that $m = \widetilde{\Omega}( n^{3/2} r )$. This bound on the error immediately implies that $|\calR'| = o(n_1 n_2 n_3)$ and so $|\calR \setminus  \calR'| = (1-o(1))n_1 n_2 n_3$. This completes the proof of the corollary.
\end{proof}

\section{Sum-of-Squares Lower Bounds}\label{sec:lb}

Here we will show strong lower bounds on the Rademacher complexity of the sequence of relaxations to the tensor nuclear norm that we get from the sum-of-squares hierarchy. Our lower bounds follow as a corollary from known lower bounds for refuting random instances of $3$-XOR \cite{G, Sch}. First we need to introduce the formulation of the sum-of-squares hierarchy used in \cite{Sch}: We will call a Boolean function $f$ a $k$-junta if there is set $S \subseteq [n]$ of at most $k$ variables so that $f$ is determined by the values in $S$.

\begin{definition}\label{def:constr1}
The $k$-round Lasserre hierarchy is the following relaxation:
\begin{enumerate}

\item[(a)] $\|v_0\|^2 = 1$, $\|v_C\|^2 = 1$ for all $C \in \calC$

\item[(b)] $\langle v_f, v_g \rangle = \langle v_{f'}, v_{g'} \rangle$ for all $f, g, f', g'$ that are $k$-juntas and $f \cdot g \equiv f' \cdot g'$

\item[(c)] $v_f + v_g = v_{f + g}$ for all $f, g$ that are $k$-juntas and satisfy $f \cdot g \equiv 0$

\end{enumerate}
\end{definition}

\noindent Here we define a vector $v_f$ for each $k$-junta, and $\calC$ is a class of constraints that must be satisfied by any Boolean solution (and are necessarily $k$-juntas themselves). See \cite{Sch} for more background, but it is easy to construct a feasible solution to the above convex program given a distribution on feasible solutions for some constraint satisfaction problem. In the above relaxation, we think of functions $f$ as being $\{0, 1\}$-valued. It will be more convenient to work with an intermediate relaxation where functions are $\{-1, 1\}$-valued and the intuition is that $u_S$ for some set $S \subseteq [n]$ should correspond to the vector for the character $\chi_S$.

\begin{definition}\label{def:constr2}
Alternatively, the $k$-round Lasserre hierarchy is the following relaxation:
\begin{enumerate}

\item[(a)] $\|u_\emptyset \|^2 = 1$, $\langle u_{\emptyset}, u_S \rangle = (-1)^{Z_S}$ for all $(\oplus_S, Z_S) \in \calC$

\item[(b)] $\langle u_S, u_T \rangle = \langle u_{S'}, u_{T'} \rangle$ for sets $S, T, S', T'$ that are size at most $k$ and satisfy $S \Delta T =  S' \Delta T'$, where $\Delta$ is the symmetric difference.

\end{enumerate}
\end{definition}

\noindent Here we have explicitly made the switch to XOR-constraints \---- namely $(\oplus_S, Z_S)$ has $Z_S \in \{0, 1\}$ and correspond to the constraint that the parity on the set $S$ is equal to $Z_S$. Now if we have a feasible solution to the constraints in Definition~\ref{def:constr1} where all the clauses are XOR-constraints, we can construct a feasible solution to the constraints in Definition~\ref{def:constr2} as follows. If $S$ is a set of size at most $k$, we define $$u_S \equiv v_g - v_f$$ where $f$ is the parity function on $S$ and $g = 1 - f$ is its complement. Moreover let $u_\emptyset = v_0$.

\begin{claim}\label{claim:red1}
$\{u_S\}$ is a feasible solution to the constraints in Definition~\ref{def:constr2}
\end{claim}

\begin{proof}
Consider Constraint $(b)$ in Definition~\ref{def:constr2}, and let $S, T, S', T'$ be sets of size at most $k$ that satisfy $S \oplus T =  S' \oplus T'$.  Then our goal is to show that
$$ \langle v_{g_S} - v_{f_S}, v_{g_T} - v_{f_T} \rangle = \langle v_{g_{S'}} - v_{f_{S'}}, v_{g_{T'}} - v_{f_{T'}} \rangle$$
where $f_S$ is the parity function on $S$, and similarly for the other functions. Then we have $f_S \cdot f_T \equiv f_{S'} \cdot f_{T'}$ because $S \oplus T =  S' \oplus T'$, and this implies that $\langle v_{f_S}, v_{f_T} \rangle = \langle v_{f_{S'}}, v_{f_{T'}} \rangle$. An identical argument holds for the other terms. This implies that all the Constraints $(b)$ hold. Similarly suppose $(\oplus_S, Z_S) \in \calC$. Since $f_S \cdot g_S \equiv 0$ and $f_S + g_S \equiv 1$ it is well-known that $(1)$ $ v_{f_S}$ and $v_{g_S}$ are orthogonal $(2)$ $v_{f_S} + v_{g_S} = v_0$ and $(3)$ since $f_S \in \calC$ in Definition~\ref{def:constr1}, we have $v_{g_S} = 0$ (see \cite{Sch}). Thus $$\langle u_\emptyset, u_S \rangle = \langle v_0, v_{g_S} \rangle -  \langle v_0, v_{f_S} \rangle = -1$$ and this completes the proof.
\end{proof}

Now following Barak et al. \cite{BBH} we can use the constraints in Definition~\ref{def:constr2} to define the operator $\widetilde{\E}[\cdot]$. In particular, given $p \in P_k^n$ where $p \equiv \sum_S c_S \prod_{i \in S} Y_i$ and $p$ is multilinear, we set $$\widetilde{\E}[p)] = \sum_S c_S  \langle u_\emptyset, u_S \rangle$$ Here we will also need to define $\widetilde{\E}[p]$ when $p$ is not multilinear, and in that case if $Y_i$ appears an even number of times we replace it with $1$ and if it appears an odd number of times we replace it by $ Y_i$ to get a multilinear polynomial $q$ and then set $\widetilde{\E}[p] = \widetilde{\E}[q]$.

\begin{claim}\label{claim:red2}
$\widetilde{\E}[\cdot]$ is a feasible solution to the constraints in Definition~\ref{def:rnorm}, and for any $(\oplus_S, Z_S) \in \calC$ we have $\widetilde{\E}[\prod_{i \in S} Y_i] = (-1)^{Z_S} $.
\end{claim}


\begin{proof}
Then by construction $\widetilde{\E}[1] = 1$, and the proof that $\widetilde{\E}[p^2] \geq 0$ is given in \cite{BBH}, but we repeat it here for completeness. Let $p = \sum_S c_S \prod_{i \in S} Y_i$ be multilinear where we follow the above recipe and replace terms of the form $Y_i^2$ with $(1/n)$ as needed.
Then $p^2 = \sum_{S,T} c_S c_T  \prod_{i \in S} Y_i \prod_{i \in T} Y_i$
and moreover
\begin{eqnarray*}
\widetilde{\E}[p^2] &=& \sum_{S, T} c_S c_T \langle u_\emptyset, u_{S \Delta T} \rangle = \sum_{S, T} c_S c_T  \langle u_S, u_{T} \rangle = \Big \| \sum_S c_S   u_S \Big \|^2 \geq 0
\end{eqnarray*}
as desired. Next we must verify that $\widetilde{\E}[\cdot]$ satisfies the constraints $\{\sum_{i = 1}^n Y_i^2 = n\}$ and $\{Y_i^2 \leq C^2\}$ for all $i \in \{1, 2, ..., n\}$, in accordance with Definition~\ref{def:pseudo-exp}. To that end, observe that
$$\widetilde{\E}\Big[\Big(\sum_{i=1}^n Y_i^2 - n \Big) q\Big] = 0$$
which holds for any polynomial $q \in P_{k - 2}^n$. Finally consider
$$\widetilde{\E}\Big[ \Big(C^2 - Y_i^2\Big) q^2\Big ] = \widetilde{\E}\Big[ \Big(C^2 - 1 \Big) q^2\Big ] \geq 0$$
which follows because $C^2 \geq 1$ and holds for any polynomial $q \in P_{\lfloor (d-d')/2 \rfloor}^n$. This completes the proof.
\end{proof}

\begin{theorem}\cite{G, Sch}
Let $\phi$ be a random $3$-XOR formula on $n$ variables with $m = n^{3/2 - \epsilon}$ clauses. Then for any $\epsilon > 0$ and any $c < 2$, the $k = \Omega(n^{c \epsilon})$ round Lasserre hierarchy given in Definition~\ref{def:constr1} permits a feasible solution, with probability $1 - o(1)$.
\end{theorem}

\noindent Note that the constant in the $\Omega(\cdot)$ depends on $\epsilon$ and $c$. Then using the above reductions, we have the following as an immediate corollary:

\begin{corollary}\label{corr:3lb}
For any $\epsilon > 0$ and any $c < 2$ and $k = \Omega(n^{c \epsilon})$, if $m = n^{3/2 - \epsilon}$ the Rademacher complexity $R^m(\| \cdot \|_{\calK_k}) =1 - o(1)$.
\end{corollary}

\noindent Thus there is a sharp phase transition (as a function of the number of observations) in the Rademacher complexity of the norms derived from the sum-of-squares hierarchy. At level six, 
$R^m(\|\cdot\|_{\calK_6})  = o(1)$ whenever $m = \omega(n^{3/2} \log^4 n)$. In contrast, $R^m(\|\cdot\|_{\calK_k})  = 1- o(1)$ when $m = n^{3/2 - \epsilon}$ even for very strong  relaxations derived from $n^{2 \epsilon}$ rounds of the sum-of-squares hierarchy. These norms require time $2^{n^{2 \epsilon}}$ to compute but still achieve essentially no better bounds on their Rademacher complexity.

  \subsection*{Acknowledgements}

 We would like to thank Aram Harrow for many helpful discussions.


\appendix

\section{Reduction from Asymmetric to Symmetric Tensors}\label{app:extensions}

Here we give a general reduction, and show that any algorithm for tensor prediction that works for symmetric tensors can be used to predict the entries of an asymmetric tensor too. Hardt gave a related reduction for the cases of matrices \cite{Ha} and it is instructive to first understand this reduction, before proceeding to the tensor case. Suppose we are given a matrix $M$ that is not necessarily symmetric. Then the approach of \cite{Ha} is to construct the following symmetric matrix:
\[ S = \left[ \begin{array}{cc}
0 & M^T \\
M & 0 \end{array} \right].\]
We have not precisely defined the notion of incoherence that is used in the matrix completion literature, but it turns out to be easy to see that $S$ is low rank and incoherent as well.

The important point is that given $m$ samples generated uniformly at random from $M$, we can generate random samples from $S$ too. It will be more convenient to think of these random samples as being generated without replacement, but this reduction works just as well without replacement too. Let $M \in \R^{n_1 \times n_2}$. Now for each sample from $S$, with probability $p = \frac{n_1^2 + n_2^2}{(n_1 + n_2)^2}$ we reveal a uniformly random entry in the either block of zeros. And with probability $1-p$ we reveal a uniformly random entry from $M$. Each entry in $M$ appears exactly twice in $S$, and we choose to reveal this entry of $M$ with probability $1/2$ from the top-right block, and otherwise from the bottom-left block. Thus given $m$ samples from $M$, we can generate from $S$ (in fact we can generate even more, because some of the revealed entries will be zeros). It is easy to see that this approach works for the case of sampling without replacement to, in that $m$ samples without replacement from $M$ can be used to generate at least $m$ samples without replacement from $S$.

Now let us proceed to the tensor case. Let us introduce the following definition, for ease of notation:

\begin{definition}
Let $m(n, r, \epsilon, \mathpzc{f}, C)$ be such that, there is an algorithm that on a rank $r$, order $d$, size $n \times n \times ... \times n$ symmetric tensor where each factor has norm at most $C$, the algorithm returns an estimate $X$ with $\mbox{err}(X) = \mathpzc{f}$ with probability $1-\epsilon$ when it is given $m(n, r, \epsilon, \mathpzc{f})$ samples chosen uniformly at random (and without replacement).
\end{definition}

\begin{lemma}
For any odd $d$, suppose we are given  $m(\sum_{j=1}^d n_j, r 2^{d-1}, \epsilon, \mathpzc{f}, \sqrt{d})$ samples chosen uniformly at random (and without replacement) from an  $n_1 \times n_2 \times ... \times n_d$ tensor $$T = \sum_{i = 1}^r a^1_i \otimes a^2_i \otimes ... \otimes a^d_i$$ where each factor is unit norm. There is an algorithm that with probability at least $1-\epsilon$ returns an estimate $Y$ with
$$\mbox{err}(Y) \leq \frac{(\sum_{j=1}^d n_j)^d}{d! 2^{d-1} \prod_{j=1}^d n_j}\mathpzc{f} $$
\end{lemma}

\begin{proof}
Our goal is to symmetrize an asymmetric tensor, and in such a way that each entry in the symmetrized tensor is either zero or else corresponds to an entry in the original tensor. Our reduction will work for any odd order $d$ tensor. In particular let
$$T = \sum_{i = 1}^r a^1_i \otimes a^2_i \otimes ... \otimes a^d_i$$
be an order $d$ tensor where the dimension of $a^j$ is $n_j$. Also let $n = \sum_{j = 1}^d n_j$. Then we will construct a symmetric, order $d$ tensor as follows. Let $\sigma_1, \sigma_2, ... \sigma_d$ be a collection of $d$ random $\pm$ variables that are chosen uniformly at random from the $2^{d-1}$ configurations where $\prod_{j = 1}^d \sigma_j = 1$. Then we consider the following random vector
$$a_i(\sigma_1, \sigma_2, ... \sigma_d) = [\sigma_1 a^1_i, \sigma_2 a^2_i, ... , \sigma_d a^d_i]$$
Here $a_i(\sigma_1, \sigma_2, ... \sigma_d) $ is an $n$-dimensional vector that results from concatenating the vectors $a^1_i, a^2_i, ..., a^d_i$ but after flipping some of their signs according to $\sigma_1, \sigma_2, ... \sigma_d$. Then we set
$$S = \E_{\sigma_1, \sigma_2, ... \sigma_d} [ \sum_{i =1}^r \Big(a_i(\sigma_1, \sigma_2, ... \sigma_d)\Big)^{\otimes d}]$$
It is immediate that $S$ is symmetric and has rank at most $2^{d-1}r$ by expanding out the expectation into a sum over the valid sign configurations. Moreover each rank one term in the decomposition is of the form $a^{\otimes d}$ where $\|a\|_2^2 = d$ because it is the concatenation of $d$ unit vectors.

If $\sigma_1, \sigma_2, ... \sigma_d$ is fixed, then each entry in $S$ is itself a degree $d$ polynomial in the $\sigma_j$ variables. By our construction of the $\sigma_j$ variables, and because $d$ is odd so there are no terms where every variable appears to an even power, it follows that all the terms vanish in expectation except for the terms which have a factor of $\prod_{j =1}^d \sigma_j$, and these are exactly terms that correspond to some permutation $\pi: [d] \rightarrow [d]$, and a term of the form $$\sum_{i=1}^d a^{\pi(1)}_i \otimes a^{\pi(2)}_i \otimes , ..., \otimes a^{\pi(d)}_i$$ Hence all of the entries in $S$ are either zero or are $2^{d-1}$ times an entry in $T$. As before, we can generate $m$ uniformly random samples from $S$ given $m$ uniformly random samples from $T$, by simply choosing to sample an entry from one of the blocks of zeros with the appropriate probability, or else revealing an entry of $T$ and choosing where in $S$ to reveal this entry uniformly at random. 
Hence:
$$\frac{1}{(\sum_{j=1}^d n_j)^d} \sum_{(i_1, i_2, ..., i_d) \in \Gamma} |Y_{i_1, i_2, ..., i_d} - S_{i_1, i_2, ..., i_d}|\leq \frac{1}{(\sum_{j=1}^d n_j)^d} \sum_{i_1, i_2, ..., i_d} |Y_{i_1, i_2, ..., i_d} - S_{i_1, i_2, ..., i_d}|$$
where $\Gamma$ represents the locations in $S$ where an entry of $T$ appears. The right hand side above is at most $\mathpzc{f}$ with probability $1-\epsilon$. Moreover each entry in $T$ appears in exactly $d!$ locations in $S$. And when it does appear, it is scaled by $2^{d-1}$. And hence if we multiply the left hand side by $$\frac{(\sum_{j=1}^d n_j)^d}{d!  2^{d-1} \prod_{j=1}^d n_j}$$ we obtain $\mbox{err}(Y)$. This completes the reduction.
\end{proof}

\noindent Note that in the case where $n_1 = n_2 = n_3 ... = n_d$, the error and the rank in this reduction increase only by at most an $e^d$ and $2^d$ factor respectively.

\end{document}